\documentclass[twoside]{article}

\usepackage[accepted]{style2022}

\usepackage{graphicx}
\usepackage{arydshln}

\usepackage[utf8]{inputenc} 
\usepackage[T1]{fontenc}    
\usepackage{hyperref}       
\usepackage{url}            
\usepackage{booktabs}       
\usepackage{amsfonts}       
\usepackage{nicefrac}       
\usepackage{microtype}      
\usepackage{bbm}
\usepackage{xcolor}         
\usepackage{subcaption}
\usepackage{float}

\usepackage[round]{natbib}

\usepackage{arydshln}
\usepackage{enumitem}
\usepackage{amsmath}
\usepackage{amsthm}
\usepackage{bbm}
\usepackage{amssymb}
\usepackage{subcaption}

\usepackage[capitalise]{cleveref}

\newtheorem{theorem}{Theorem}
\newtheorem{lemma}[theorem]{Lemma}

\DeclareMathAlphabet\mathbfcal{OMS}{cmsy}{b}{n}

\newcommand{\R}{\mathbb{R}}

\newcommand{\N}{\mathbb{N}}

\newcommand{\E}{\mathbb{E}}
\newcommand{\hist}{\mathcal{S}}

\newcommand{\Prob}{p}
\newcommand{\q}{q}

\newcommand{\prob}{\Prob}
\newcommand{\var}{\text{Var}}

\newcommand{\ind}{\mathbbm{1}}
\newcommand{\sep}{\!\;|\;\!}
\newcommand{\hit}{\text{hit}}
\newcommand{\vocab}{\mathbb{M}}

\begin{document}

\twocolumn[

\mainpapertitle{Probabilistic Querying of Continuous-Time Event Sequences}

\mainauthor{Alex Boyd${}^1$ \And Yuxin Chang${}^2$ \And  Stephan Mandt${}^{1,2}$ \And Padhraic Smyth${}^{1,2}$}

\mainaddress{ ${}^1$Department of Statistics\quad ${}^2$Department of Computer Science \\ University of California, Irvine \\ \texttt{\{alexjb,yuxinc20,mandt,p.smyth\}@uci.edu}} 

\runningauthor{Probabilistic Querying of Continuous-Time Event Sequences}
]

\begin{abstract}
Continuous-time event sequences, i.e., sequences consisting of continuous time stamps and associated event types (``marks''), are an important type of sequential data with many applications, e.g., in clinical medicine or user behavior modeling. Since these data are typically modeled autoregressively (e.g., using neural Hawkes processes or their classical counterparts), it is natural to ask questions about future scenarios such as ``what kind of event will occur next'' or ``will an event of type $A$ occur before one of type $B$''. Unfortunately, some of these queries are notoriously hard to address since current methods are limited to naive simulation, which can be highly inefficient. 
This paper introduces a new typology of query types and a framework for addressing them using importance sampling. Example queries include predicting the  $n^\text{th}$ event type in a sequence and the hitting time distribution of one or more event types. 
We also leverage these findings further to be applicable for estimating general ``$A$ before $B$'' type of queries.
We prove theoretically that our estimation method is effectively always better than naive simulation and show empirically based on three real-world datasets that it is on average 1,000 times more efficient than existing approaches.
\end{abstract}

\section{Introduction}

Continuous-time event data occurs across a wide range of applications and areas such as user behavior modeling \citep{mishra2016feature, kumar2019predicting}, finance \citep{bacry2012non, hawkes2018hawkes}, and healthcare \citep{nagpal2021deep, chiang2022hawkes}. The data typically consists of sets of variable-length sequences where each sequence is a set of ordered events, and each event is associated with a continuous time-stamp and a categorical event type. Such data are often modeled as marked temporal point processes (MTPPs), and a broad variety of  modeling frameworks have been successfully developed both in the statistical literature (e.g., Hawkes processes \citep{hawkes1971spectra}) and in the machine learning literature (e.g., neural MTPP models \citep{mei2017neural}.
These MTPP modeling frameworks provide a general and flexible setup for making one-step-ahead predictions  such as the timing and/or type of the next event time, conditioned on a partial history of sequence.

In this paper, we look beyond one-step ahead predictions for multiple event types and instead investigate how to efficiently answer queries that involve more complex statements about future events and their timing. Such queries include hitting time queries (``what is the probability that at least one event of type $A$ will occur before time $t$''), queries of the form ``what is the probability that $A$ will occur  before $B$,'' 
as well as computing the marginal distribution of  event types for the $n^{\text{th}}$ next event (irrespective of time). These types of queries are useful across a variety of applications, such as making predictions conditioned on a patient's medical and treatment history, or conditioned on a customer's page view and purchase history.

However, exact computation of such queries is intractable in general except in the case of simple parametric models. For a standard MTPP model to directly answer such queries requires that all intervening events (from current time to the event(s) of interest in the query) are marginalized over. In particular, this involves marginalizing over both the combinatorially-large space of possible  event types as well as the uncountably infinite space of possible event timings. While direct simulation of future trajectories from a model provides one avenue for answering such queries (e.g., see \cite{daley2003introduction}) these ``naive'' methods can be very inefficient (both statistically and computationally), as we will demonstrate later in the paper. More efficient alternative approaches (to the naive simulation method) appear to be completely unexplored (to our knowledge), for both neural and non-neural MTPP models.

We develop a general query framework based on importance sampling that enables efficient estimates of various types of queries. In our approach, we first transform each query into unified forms and then derive the marginal distribution of interest as functions of type-specific intensities (expected instantaneous rates of occurrence). Our proposed novel marginalization scheme empowers real-time computation of previously intractable probabilistic queries, with proven higher efficiency compared to naive estimates. 
Furthermore, experiments on three real-world datasets in different domains demonstrate that our proposed estimation method is significantly more efficient than the naive estimate in practice. For example, for hitting time queries with neural Hawkes processes, we show an average magnitude of $10^3$ reduction in estimator variance.

Our approach for answering probabilistic queries is general-purpose in the sense that it can be integrated with any intensity-based black-box MTPP models, both parametric and neural. To summarize, our main contributions are:
\begin{itemize}[itemsep=0.0em,leftmargin=0.9em,topsep=0em]
\item We identify and formalize a general class of probabilistic queries that cover a wide range of queries of interest, such as the distribution of the first occurrence of certain event types (hitting times), the $n^{\text{th}}$ occurring event type irrespective of time (marginal mark queries), or queries addressing the order of event types (``A before B'' queries).
\item Within this class of queries, we develop a novel proposal distribution for importance sampling. This distribution is easy to sample from, simple to evaluate likelihoods with, and results in \emph{guaranteed} increases in efficiency when compared to existing estimation techniques.
\item We evaluate our proposed estimation technique across three real-world user behavior datasets, as well as in an artificial setting. In all cases, we find dramatic reductions in estimator variance compared to existing methods---often times by several orders of magnitude.
\end{itemize}

\section{Related Work}
A large variety of MTPP models have been developed over recent decades, aimed at modeling sequences of marked event data with varying sorts of behaviors. This behavior has been both explicitly modeled with parametric MTPP models \citep{isham1979self, daley2003introduction}, and implicitly using neural network-based methods \citep{du2016recurrent, bilovs2019uncertainty, shchur2019intensity, enguehard2020neural, zuo2020transformer, deshpande2021long}. Notably in these categories are the self-exciting Hawkes process \citep{hawkes1971spectra, liniger2009multivariate} and the neural Hawkes process \citep{mei2017neural} respectively. The majority of neural MTPP models utilize some form or extension of recurrent neural networks to model conditional intensity functions (or equivalent transformations thereof).
MTPP models have been broadly applied to next event prediction across a number of different application areas: seismology \citep{ogata1998space}, finance \citep{bacry2012non, hawkes2018hawkes}, social media behavior \citep{mishra2016feature, rizoiu2017expecting}, and medical outcomes \citep{10.2307/2985181, andersen2012statistical}.\footnote{Survival analysis is a special case of temporal point processes where the event of interest can only occur once.} Neural-based methods have also been successful at more novel tasks such as imputing missing data \citep{shchur2019intensity, mei2019imputing, gupta2021learning} and sequential representation learning \citep{shchur2019intensity, NEURIPS2020_f56de5ef}. 

However, to the best of our knowledge, apart from the naive sampling approach (e.g., \cite{daley2003introduction}), there is no existing work on answering general probabilistic queries (such as hitting time of a collection of event types) with MTPP models, which is the focus of this paper. That being said, estimating these queries for discrete time has been investigated \citep{boyd2022predictive}. While there does not exist an easy mapping of those techniques onto continuous time, this previous work will serve as a large source of inspiration for what we propose in this paper. 

\section{Preliminaries}
\subsection{Notation for Event Sequences}
Let $\tau_1, \tau_2, \dots \in \R_{\geq 0}$ be a sequence of continuous random variables with the constraint that $\forall_i: \tau_i < \tau_{i+1}$. These represent the time of occurrence for events of interest. Each event has an associated categorical value, such as a label or a location, that is referred to as a \emph{mark}. An event is jointly represented as (i) a time of occurrence $\tau_i$ and (ii) an associated mark random variable $\kappa_i\in\vocab$.  
In this work we will focus on the finite discrete setting of a fixed vocabulary for marks: $\vocab=\{1,2,\dots,K\}$, although more generally the mark space $\vocab$ can be defined on a variety of different domains.
\\\\
Let the \emph{sequence} of events over a specified time range $[a,b]\subset\R_{\geq 0}$ be denoted as 
$$\hist[a,b]=\{(\tau_i,\kappa_i)\sep \tau_i \in [a,b] \text{ for } i\in\N\}.$$
with similar definitions for $\hist(a,b]$ and $\hist[a,b)$.
For simplicity, we will let $\hist(t)$ be shorthand for $\hist[0,t)$ such that $\hist(\tau_i)=\{(\tau_1,\kappa_1),\dots,(\tau_{i-1},\kappa_{i-1})\}$.\footnote{Note that the majority of the point process literature refers to this sequence as a \emph{history} of events and is represented via $\mathcal{H}$. We forgo this traditional terminology and notation to emphasize that our work is primarily about estimating queries for \emph{future} events.} 
We will use $\hist_k$ to refer to \emph{mark-specific} sequences, i.e., $\hist_k(t) = \{(\tau_i,\kappa_i) \in \hist(t) \sep \kappa_i=k\}$.

\subsection{Marked Temporal Point Processes}
The generative mechanism for these point patterns are generally referred to as \emph{marked temporal point processes} (MTPPs). 
MTPP models are capable of approximating the distribution of a given sequence of $N$ events, $\prob(\hist[0,\tau_N])$.\footnote{For brevity and consistent notation, we will be using $\prob(\cdot)$ in reference to both probability densities and masses when appropriate.} These models are typically constructed in an autoregressive fashion,
\begin{align}
\prob(&\hist[0,\tau_N])=\prod_{i=1}^{N}\prob(\tau_i,\kappa_i\sep\hist[0,\tau_{i-1}]),
\end{align}
where the distribution for the next event $(\tau_i,\kappa_i)$ conditioned on the preceding terms is  modeled with the expected instantaneous rate of change for each mark. This is referred to as the \emph{marked intensity function} and is defined formally as
\begin{align}
\lambda_k(t\sep\hist(t))dt:=\E_p\left[\ind(|\hist_k[t,t+dt)|=1)\sep \hist(t)\right].
\end{align}
For brevity, we typically use the following $*$ convention to suppress the conditional: $\lambda^*_k(t):=\lambda_k(t\sep\hist(t))$. Note that these functions not only condition on the preceding events, but also on the fact that no events have occurred since the last event up until time $t$, i.e., $\prob(\cdot\sep\hist[0,t)) \neq \prob(\cdot\sep\hist[0,\tau_{i-1}])$.

The total intensity function, $\lambda^*(t):=\sum_{k\in\vocab}\lambda^*_k(t)$, is sufficient to describe the timing of the next event $\tau_i$. The distribution of the mark conditioned on the timing of the next event is naturally described as $\prob(\kappa_i=k\sep\tau_i=t)\equiv \frac{\lambda^*_k(t)}{\lambda^*(t)}$. We will be assuming that the native output of any model we are working with will produce a vector of marked intensity functions over the mark space $\vocab$ evaluated at time $t$. Any MTPP with a defined set of marked intensity functions can be easily sampled from by utilizing a thinning procedure \citep{ogata1981lewis}, if not directly.

Lastly, the likelihood of a given sequence $\hist$ of length $N$ over an observation window $[0,T]$ can be computed in terms of intensity values:
\begin{align}
\prob(\hist[0,T])=\left(\prod_{i=1}^{N}\lambda_{\kappa_i}^*(\tau_i)\right)\exp\left(-\int_0^T \lambda^*(s)ds\right)\hspace{-0.2em}.
\end{align}

\section{Querying MTPPs}

We are interested in evaluating probabilistic statements, or rather \emph{queries}, on a trained MTPP model $p$. 
Furthermore, we are  interested in evaluating queries that are conditioned on a partially observed sequence (e.g., ``what is the likelihood that at least one event of type $A$ will occur in the next year given a patient's medical history?''). 

Formally, we define a probabilistic query as a probability statement of the form
\begin{align}
\prob(\hist \in \mathcal{Q}) \text{ where } \mathcal{Q} \subset \Omega_{\hist}\equiv \text{Sample Space of } \hist,
\end{align}
where $p$ is the model's distribution over future event sequences. We refer to $\mathcal{Q}$ as the \emph{query space}. The contents of the query space naturally will vary depending on the query at hand. An example query space and associated elements can be seen in \cref{fig:example_query_space}. It is worth noting that in most contexts, the cardinality of $\mathcal{Q}$ will be uncountably infinite. 

\begin{figure}
    \centering
    \includegraphics[width=\columnwidth]{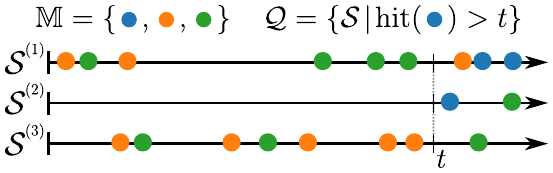}
    \caption{Example query space $\mathcal{Q}$ for the hitting time (first occurrence) of blue marks being greater than some time $t$. Sequences shown, $\hist^{(i)}$, all belong to the query space as they each do not contain a blue event occurring before time $t$.}
    \label{fig:example_query_space}
\end{figure}

This section will begin by discussing what probabilistic queries 
are readily available and tractable for a given model. Following this, we will present a novel class of queries, of which include hitting time and marginal mark queries, as well as an importance sampling estimation procedure. Finally, we will  discuss ``A before B'' queries and how to efficiently estimate them under our novel framework. \emph{Without loss of generality, we suppress the notation for conditioning on partially observed sequences and present all derivations and notations for unconditional queries.}

\subsection{Directly Tractable Queries}
Due to the model's autoregressive nature, queries about the immediate next event of a sequence are the only types of queries that can be directly evaluated without marginalization. We will now present the two main types for MTPPs.

\paragraph{Marginal Distribution of Next Timing: $\tau_1$}
In general, it can be shown that $\lambda^*(t)=\frac{f_{\tau_i}(t \sep \hist[0,\tau_{i-1}])}{1-F_{\tau_i}(t \sep \hist[0,\tau_{i-1}])}$ where $t \in (\tau_{i-1},\tau_i]$, $f_{\tau_i}$ is the probability density function (PDF) of $\tau_i$, and $F_{\tau_i}$ is the and cumulative density function (CDF) of $\tau_i$. By recognizing that $\lambda^*(t)=\frac{d}{dt}\log(1-F_{\tau_i}(t\sep \hist[0,\tau_{i-1}]))$, we find that the CDF of the next event timing $\tau_1$ is
\begin{align}
\prob(\tau_1 \leq t) := F_{\tau_1}(t) = 1 - \exp\left(-\int_0^t \lambda^*(s)ds\right).\label{eq:first_time_cdf}
\end{align}
Differentiating this result with respect to $t$ yields the PDF: $f_{\tau_1}(t)=\lambda^*(t)\exp\left(-\int_0^t\lambda^*(s)ds\right)$.
Note that we only immediately have access to the analytical form of the first future event timing $\tau_1$. To achieve the same results for $\tau_i$ in general would require marginalizing over all $i-1$ events which is rather cumbersome to do exactly.

\paragraph{Marginal Distribution of Next Mark: $\kappa_1$} 
Let $A\subset\vocab$. It follows then that the probability of the first event having a mark in $A$ is computed as follows:
\begin{align}
\prob(\kappa_1 \in A) & = \int_0^\infty \prob(\kappa_1 \in A \sep \tau_1 = t)f_{\tau_1}(t)dt\nonumber \\
& = \int_0^\infty \frac{\lambda_A^*(t)}{\lambda^*(t)}\lambda^*(t)\exp\left(-\int_0^t \lambda^*(s)ds\right)dt\nonumber \\
& = \int_0^\infty \lambda_A(t)\exp\left(-\int_0^t \lambda^*(s)ds\right)dt,\label{eq:first_marg_mark}
\end{align}
where $\lambda^*_A(t)=\sum_{k\in A}\lambda^*_k(t)$. Replacing the outer integration bounds of $[0,\infty)$ with $[a,b]$ gives the joint query $\prob(\tau_1\in[a,b], \kappa_1\in A)$.

Both of these different queries can potentially be computed analytically if the form of $\lambda^*$ permits, otherwise they can be estimated using approximate integration techniques.

\subsection{Naive Estimation of Queries}
When considering more complex queries, for example those that deal with sequences of events or those far in the future, it becomes necessary to have to rely on simulating potential trajectories in order to help estimate their values. This is due to the fact that exactly representing a probabilistic query in terms of intensity values would need to involve many nested integrals (for each potential interim event), potentially an infinite amount of them depending on the query.

The de facto method for approximating arbitrary probabilistic queries involves generating sequences and observing the frequency at which the query condition is met \citep{daley2003introduction}. This can be seen as a Monte Carlo estimate to the following formulation:
\begin{align}
\prob(\hist(T)\in\mathcal{Q}) = \E_{\prob}\left[\ind(\hist(T)\in\mathcal{Q})\right],
\end{align}
where $\ind(\cdot)$ is the indicator function. We refer to this procedure as ``naive'' estimation because this does not take into account any information about the query when sampling. 

\subsection{General Restricted-Mark Queries} \label{sec:gen_mark_queries}
One way to improve upon the naive procedure is to leverage information about the query in a proposal distribution in conjunction with importance sampling.
To do so though, we must first constrain ourselves to a specific class of query being considered. Additionally, this class of interest should take into account different aspects of sampling sequences using MTPPs for our proposal distribution. Namely, these models can easily be forced to \emph{not} sample events of specific types over a period of time (e.g., set $\lambda^*_A(t):=0$ for some time interval). Conversely, it is not immediately obvious how to \emph{encourage} or \emph{force} an event to occur within a specified time range.

As such, a natural class of queries can be seen in which over one or more specified spans of time we restrict what types of events are allowed and not allowed to occur. We term this class as ``general restricted-mark queries.''

We will now more formally define this class of queries. Consider positive real values $\alpha_1,\dots,\alpha_n$ such that $\alpha_i < \alpha_{i+1}$. These values naturally split the timeline $\R_{\geq 0}$ into $n+1$ spans: $[0,\alpha_1], (\alpha_1, \alpha_2], \dots, (\alpha_{n-1},\alpha_{n}], (\alpha_n, \infty)$.  Furthermore, let $\mathcal{M}_i \subseteq \vocab$ for $i=1,\dots,n$ represent restricted mark spaces for the first $n$ spans. The class of queries is concerned with how likely sequences spanning $[0,\alpha_n]$ respect the restricted mark spaces in each associated interval: \begin{align}
&\prob\left(\cup_{i=1}^n \{\text{No events with types } \mathcal{M}_i \text{ in } t\in(\alpha_{i-1},\alpha_i]\}\right) \nonumber \\
& = \prob\left(\land_{i=1}^n \forall_{(\tau, \kappa) \in \hist(\alpha_{i-1}, \alpha_i]} \kappa \notin \mathcal{M}_i \right) \text{ with } \alpha_0=0.
\end{align}
See \cref{fig:example_query} for an illustrated example query. 
This is a very flexible class of queries that includes many meaningful individual queries, as will be further discussed in \cref{sec:complex_queries}.

\begin{figure}
    \centering
    \includegraphics[width=\columnwidth]{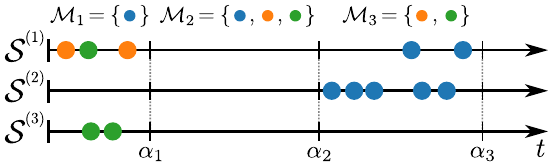}
    \caption{Three potential sequences $\hist^{(i)}$ that satisfies the condition of an example restricted-mark query. The mark space $\mathbb{M}$ in this context is equivalent to that in \cref{fig:example_query_space}.}
    \label{fig:example_query}
\end{figure}

\paragraph{Importance Sampling and Proposal Distribution}
Let $\q$ be a proposal distribution with support over at least the intersection of the support of $\prob$ and the query space $\mathcal{Q}$ (i.e., $\text{supp}(\q) \supseteq \text{supp}(\prob)\cap \mathcal{Q}$). It then follows that 
\begin{align}
\E_{\prob}\left[\ind(\hist(T)\in\mathcal{Q})\right] = \E_{\q}\left[\ind(\hist(T)\in\mathcal{Q})\frac{\prob(\hist(T))}{\q(\hist(T))}\right]. \label{eq:imp_sampling_general}
\end{align}
It can be shown that the optimal proposal distribution (i.e., lowest estimator variance) takes the form \citep{robert1999monte}:
\begin{align}
\q_\text{optimal}(\hist(T)) &:= \frac{|\ind(\hist(T)\in\mathcal{Q})|\prob(\hist(T))}{\E_{\prob}[|\ind(\hist(T)\in\mathcal{Q})|]} \\
& = \prob(\hist(T)\sep\hist(T)\in\mathcal{Q}),
\end{align}
however, this is not immediately usable since it involves computing the exact query that we are trying to estimate in the first place. 

The more our actual proposal distribution $q$ resembles $q_\text{optimal}$, the more efficient our estimation procedure will be. 
Since conditioning on future events is difficult for neural autoregressive models, we can 
instead only apply immediate ``local'' restrictions on the trajectory such that a sequence will remain within $\mathcal{Q}$. This can be accomplished by letting $\q$ be a MTPP with intensity
\begin{align}
\mu^*_k(t)=\ind(k\notin \mathcal{M}_i)\lambda^*_k(t)
\end{align}
for $k\in\mathcal{M}$ and $t\in(\alpha_{i-1}, \alpha_i]$. It should be noted that this can be seen as the natural extension of the proposal distribution in \citet{boyd2022predictive} to continuous time. This naturally leads to the likelihood of any sequence generated under $\q$ as being
\begin{align}
& \q(\hist[0,T]) =  \left(\prod_{i=1}^{N}\mu^*_{\kappa_i}(\tau_i)\right)\exp\left(-\int_0^T \mu^*(s)ds\right) \nonumber \\
& = \left(\prod_{i=1}^{N}\lambda^*_{\kappa_i}(\tau_i)\right)\exp\left(-\sum_{i=1}^n \int_{\alpha_{i-1}}^{\alpha_i} \lambda^*_{\mathbb{M}\setminus \mathcal{M}_i}(s)ds\right) \hspace{-0.2em}
\end{align}
where $N=|\hist[0,T]|$. 
This proposal distribution was exactly constructed so that every sample generated will always belong to the query space.
Applying this to \cref{eq:imp_sampling_general} yields
\begin{align}
\prob(\hist(T)\in\mathcal{Q}) & = \E_{\q}\left[\exp\left(-\sum_{i=1}^n \int_{\alpha_{i-1}}^{\alpha_i} \lambda^*_{\mathcal{M}_i}(s)ds\right)\right].\label{eq:imp_sampling_specific}
\end{align}
Any query in this class can now be estimated in an unbiased fashion by using Monte Carlo estimation on \cref{eq:imp_sampling_specific}.

\paragraph{Estimator Efficiency}
Since both the naive and importance sampled estimators are unbiased, whichever has lower variance can be seen as the more \emph{efficient} estimator. 

Assume that $\mathcal{Q}$ belongs to a general restricted-mark query and that $\pi=\prob(\hist(T)\in\mathcal{Q})$. Let
\begin{align}
\hat{\pi}_\text{Naive}(\hist(T)) & = \ind(\hist(T) \in \mathcal{Q}), \\
\hat{\pi}_\text{Imp.}(\hist(T)) & = \exp\left(-\sum_{i=1}^n \int_{\alpha_{i-1}}^{\alpha_i} \lambda_{\mathcal{M}_i}^*(s)ds \right),
\end{align}
where both are unbiased estimators of $\pi$ under $p$ and $q$ respectively.
Note that $\hat{\pi}_\text{Imp.}(\cdot)\in[0,1]$ as $\lambda^*_k(\cdot) \geq 0$. Finally, let relative efficiency of the two estimators be defined as
\begin{align}
\text{eff}(\hat{\pi}_\text{Imp.}, \hat{\pi}_\text{Naive}):=\frac{\var_{\prob}\left[\hat{\pi}_\text{Naive}(\hist(T))\right]}{\var_{\q}\left[\hat{\pi}_\text{Imp.}(\hist(T))\right]}.
\end{align}

\begin{theorem}
If $\pi \in (0,1)$ and $\lambda^*(t) < \infty$ for all $t \in [0,T]$, then $\text{eff}(\hat{\pi}_\text{Imp.}, \hat{\pi}_\text{Naive}) > 1$. In other words, under these conditions $\hat{\pi}_\text{Imp.}$ is \underline{always} more efficient than $\hat{\pi}_\text{Naive}$. \label{thm:eff}
\end{theorem}
\begin{proof}
Since the naive estimator is unbiased and binary, then it follows that $\hat{\pi}_\text{Naive}(\hist(T)) \sim \text{Bern}(\pi)$. Thus, $\var_{\prob}\left[\hat{\pi}_\text{Naive}(\hist(T))\right]=\pi-\pi^2$.

To approach the variance of the importance sampling estimator, we note that
\begin{align}
\var_{\prob}\left[\hat{\pi}_\text{Imp.}(\hist(T))\right] & = \E_{\q}\left[\hat{\pi}^2_\text{Imp.}\right] - \E_{\q}\left[\hat{\pi}_\text{Imp.}\right]^2 \nonumber \\
& = \E_{\q}\left[\hat{\pi}^2_\text{Imp.}\right] - \pi^2 \nonumber \\
& \leq \E_{\q}\left[\hat{\pi}_\text{Imp.}\right] - \pi^2 \text{ since } \hat{\pi}_\text{Imp.}\in[0,1] \nonumber \\
& = \pi-\pi^2  
\end{align}
The equality only holds if $\pi\in\{0,1\}$ or $\hat{\pi}_\text{Imp.} \sim \text{Bern}(\pi)$. The latter condition is due to the fact that for $[0,1]$ bounded random variables with mean $\pi$, if the variance is equal to $\pi - \pi^2$ then this implies it is Bernoulli (see Appendix for proof). However, when $\pi \in (0,1)$ then unless $\lambda^*(t) = \infty$ for some subset of $[0,T]$ it is impossible for $\hat{\pi}_\text{Imp.}(\hist(T))$ to equal $0$. Thus, outside of those circumstances the inequality is strict and $\text{eff}(\hat{\pi}_\text{Imp.},\hat{\pi}_\text{Naive}) > 1$.
\end{proof}

\subsection{Practical Estimation of Complex Queries} \label{sec:complex_queries}
We will now apply our findings from \cref{sec:gen_mark_queries} to produce estimators for three different complex, probabilistic queries.

\paragraph{Marginal Distribution of Hitting Time: $\hit(A)$}
Let $A\subset \mathbb{M}$ and $A \neq \emptyset$. The first occurrence of an event with type $k\in A$, regardless of events of other types, is referred to as the \emph{hitting time} of $A$ or $\hit(A)$. The probabilistic query of the CDF of the hitting time of $A$ at a specific time $t$ can be seen as a query under the general restricted-mark class:
\begin{align}
\prob(&\hit(A) \leq t) = 1 - \prob(\hit(A) > t) \nonumber \\
& = 1 - \prob(\{\text{No events of types } A \text{ in } [0,t] \}) \nonumber \\
& = 1 - \prob(\forall_{(\tau,\kappa)\in\hist[0,t]}\kappa\notin A) \nonumber \\
& = 1- \E_{\q}\left[\exp\left(- \int_{0}^t \lambda^*_{A}(s)ds\right)\right]. \label{eq:hitting_time_is}
\end{align}
Interestingly, the importance sampled result of this query greatly resembles the CDF of the general first event timing: $F_{\tau_1}(t)=1-\exp\left(-\int_0^t \lambda^*(s)ds\right)$.\footnote{It is important to remember that in the general case, we must marginalize over possible trajectories for other types of events $A'$ as these can all either potentially influence the intensity of events of type $A$.} Furthermore, should $A=\mathbb{M}$ then we recover $F_{\tau_1}(t)$ as the estimator becomes deterministic (due to $\mu^*(t)=0 \implies q(\hist)\propto \ind(\hist=\emptyset)$).

\paragraph{Marginal Distribution of $n^\text{th}$ Mark: $\kappa_n$}
Let $A\subset \mathbb{M}$ and $n\geq 1$. The distribution of the marginal $n^\text{th}$ mark describes how likely it is that the $n^\text{th}$ event has a mark $k\in A$, irrespective of the timing of itself or of any of the $n-1$ events that occurred prior. In contrast to hitting time queries, we do not fix the integration bounds but rather sample them to be the timings of the $\tau_{n-1}$ and $\tau_n$. In doing so, this query falls under the general mark-restricted framework:
\begin{align}
\prob(\kappa_n\in A) & = \prob(\{\text{No events of types } A' \text{ in } (\tau_{n-1},\tau_{n}]\}) \nonumber \\
& = \prob(\forall_{(\tau,\kappa)\in \hist(\tau_{n-1},\tau_n]}\kappa\notin A') \nonumber \\
& = \E_{\q}\left[\exp\left(-\int_{\tau_{n-1}}^{\tau_n}\lambda_{A'}^*(s)ds\right)\right]
\end{align}
where $A'=\mathbb{M}\setminus A$. Interestingly, we can also compute the complement under the same framework as $\prob(\kappa_n\in A) = 1-\E_{\q}\left[\exp\left(-\int_{\tau_{n-1}}^{\tau_n}\lambda_{A}^*(s)ds\right)\right]$.

\paragraph{``$A$ before $B$'' Queries}
The last class of queries we will discuss are what we refer to as ``$A$ before $B$'' queries. To be precise, we are interested in the probability of an event with some type $k\in A$ occurring before an event with some type $k\in B$ where $A\cap B=\emptyset$ and non-empty $A,B\subset\mathbb{M}$. In math, this is formally represented as $\prob(\hit(A) < \hit(B))$.

Surprisingly, with our previous developments we can actually estimate this query using importance sampling in conjunction with proposal distribution $\q$. For the proposal distribution, let $\mu_k^*(t)=\ind(k \notin A\cup B)\lambda^*_k(t)$. It then can be shown that
\begin{align}
\prob&(\hit(A)<\hit(B)) \nonumber \\
& = 1-\E_{\q}\left[\int_0^\infty \lambda_B^*(t) \exp\left(-\int_0^t \lambda^*_{A\cup B}(s) ds\right)dt\right] \nonumber \\
& = \E_{\q}\left[\int_0^\infty \lambda_A^*(t) \exp\left(-\int_0^t \lambda^*_{A\cup B}(s) ds\right)dt\right] \label{eq:a_before_b}
\end{align}
with both expressions being equal due to the complement $1-\prob(\hit(A)>\hit(B))$ also being estimable under this derivation. See the Appendix for derivations. 

Interestingly, just like the parallels between the hitting time CDF and the first event time CDF, there exist similar comparisons for \cref{eq:a_before_b} and the analytical form of the marginal distribution for the first mark $\prob(\kappa_1\in A)=\int_0^\infty \lambda^*_A(t)\exp\left(-\int_0^t \lambda^*(s)\right)dt$. Additionally, should $B=A'$ then the estimator becomes deterministic and we recover the form of $\prob(\kappa_1\in A)$. 

Note that the expectations in \cref{eq:a_before_b} are with respect to $\hist(\infty)\sim \q$, which is naturally not possible to evaluate; however, since the integrands are non-negative we can compute natural lower and upper bounds by sampling $\hist(T)\sim\q$ and integrating over $[0,T]$ instead of $[0,\infty)$. Lastly, since these bounds utilize the same proposal distribution, we can actually compute both at the same time for a little extra computation. It then stands to reason that a good estimate for $\prob(\hit(A) < \hit(B))$ would be an average of the upper and lower bounds:
\begin{align}
& \prob(\hit(A) < \hit(B)) \approx \label{eq:ab_estimator}\\
&  \frac{1}{2} + \E_{\q}\hspace{-0.3em}\left[\int_0^T \frac{\lambda_A^*(t)-\lambda^*_B(t)}{2} \exp\left(\hspace{-0.1em}-\hspace{-0.1em}\int_0^t \lambda^*_{A\cup B}(s) ds\hspace{-0.15em}\right)\hspace{-0.15em}dt\right],\nonumber
\end{align}
where $T>0$ can either be set as a constant or could be dynamically determined on a per sequence basis based on some precision threshold.
Since $T$ is truncated, this estimate is no longer unbiased.

\section{Experiments} \label{sec:experiments}

We investigate the effectiveness of our novel importance sampling regime in the context of estimating hitting time, ``A before B,'' and marginal mark distribution queries, while conditioning on partially observed sequences. We find that across both synthetic and real settings as well as parametric and neural-network-based models that our importance sampling estimator dramatically reduces variance compared to naive sampling and results in a much lower error on average. Furthermore, we demonstrate that, on average, these gains in performance outweigh any potential increases in computation time.

\paragraph{Ground Truth}
Computation of any arbitrary query $\prob(\hist(T)\in\mathcal{Q})$ to arbitrary precision is  intractable in the general case. Given this, in our experiments we compute our queries with an unbiased estimator to high precision  using a large amount of computation, with much higher precision than any of the methods and scenarios evaluated for a given experiment. We refer to the result of this high-precision computation as ``ground truth'' below.  

\paragraph{Metrics of Interest}
There are two primary metrics with which we judge query estimation procedures: mean relative absolute error and relative efficiency (or variance reduction should one of the estimators be biased). The former is defined as the mean of $|\pi-\hat{\pi}|/\pi$, where $\pi=\prob(\hist(T)\in\mathcal{Q})$ and $\hat{\pi}$ is some estimator of $\pi$, over different queries (and potentially models). This particular form of error is chosen to offset the fact that $\pi\in[0,1]$, which can lead to naturally closer estimates should $\pi$ be close to $0$ or $1$. The latter metric of interest is the relative efficiency (or variance reduction) of importance sampling compared to naive sampling. This is calculated by dividing the variance of the naive estimator (calculated using ground truth: $\pi(1-\pi)$) with the variance of the importance sampled estimator (calculated empirically). As an example, a value of 5 for this metric indicates that, on average, 5 times as many samples are needed for naive estimation to achieve an estimator variance as low as that of importance sampling.

\subsection{Real-world Experiments}

\paragraph{Datasets}

\begin{table}
\centering
\caption{Real-world Dataset Summary Statistics}
\begin{tabular}{l @{\hspace{0.8cm}} c c c}
\toprule
Dataset &  \# Sequences & $T_{max}$ & \# Marks\\
\midrule
MovieLens & 34,935 & 43,000 & 182 \\
MOOC & 6,863 & 715 & 97 \\
Taobao & 17,777 & 192 & 1,000 \\
\bottomrule
\end{tabular}
\label{tab:datasets}
\end{table}

We conduct our real-world experiments on three sequential user-behavior datasets. In all three, a sequence is defined as the records pertaining to a single individual.
The \textbf{MovieLens 25M} dataset \citep{harper2015movielens} contains records of user-generated movie reviews alongside a rating. Marks represent pairs of categories under which a reviewed movie can be classified.
The \textbf{MOOC} dataset \citep{kumar2019predicting} is a collection of online user-behaviors for students taking an online course. Marks represent the type interaction a student has performed.
Lastly, the \textbf{Taobao} user behavior dataset \citep{zhu2018learning} contains page-viewing records from users on an e-commerce platform. Marks are defined as the category of the item being viewed, with categories outside of the top 1,000 discarded.
All datasets were split into 75\% training, 10\% validation, and 15\% test splits for model fitting and experiments.
Summary statistics for these datasets can be found in \cref{tab:datasets}. All preprocessing details for these three datasets can be found in the Appendix.

\paragraph{Models}
All real-world experiments utilize neural Hawkes models \citep{mei2017neural}, one trained for each dataset. Each model was trained to convergence on the training split with stability/generality ensured via the validation split. All training and model details can be found in the Appendix.

\paragraph{Hitting Time Queries}
For each dataset, we randomly sample 1,000 different sequences $\hist(T)$. For each sequence, we condition on the first five events, $\hist[0,\tau_5]$, and evaluate a hitting time query for the remaining future.\footnote{All experiments evaluate necessary integrals with the trapezoidal rule. For more details, see Appendix.} The specific hitting time query asked is $\prob(\hit(k)\leq t\sep\hist[0,\tau_5])$ where $k:=\kappa_6$ and $t:=10\times\tau_6$ for $(\tau_6,\kappa_6)\in\hist(T)$. 

We compared estimating this query with naive sampling and importance sampling using varying amounts of samples: $\{2,4,10,25,50,250,1000\}$. Mean RAE compared to ground truth (estimated using importance sampling with 5000 samples) can be seen in \cref{fig:hit_err}. We witness roughly an order of magnitude of improvement in performance for the same amount of samples. Primarily, we attribute this improvement to the fact that naive sampling only collects binary values, whereas our proposed procedure collects much more dense information over the entire span $[\tau_5, t]$.

We also analyze the relative efficiency of our estimator compared to naive sampling. For each query asked, the efficiency was estimated using 5000 importance samples. The results can be seen in \cref{fig:hit_eff}.  
We achieve a dramatic decrease in variance by several orders of magnitude, in the majority of contexts, across all datasets. Interestingly, it appears that the efficiency is correlated with the underlying ground truth value $\pi$. We believe this may be due to the form of the importance sampling estimator: $1-\exp\left(-\int_0^t \lambda_k^*(s)ds\right)$. Since the intensity function is non-negative, it is simple for the model to produce estimates close to 0; however, to producing values close to 1 requires the integral to tend towards infinity. 

\begin{figure}[!t]
    \centering{\phantomsubcaption\label{fig:hit_err}\phantomsubcaption\label{fig:hit_eff}}
    \includegraphics[width=\columnwidth]{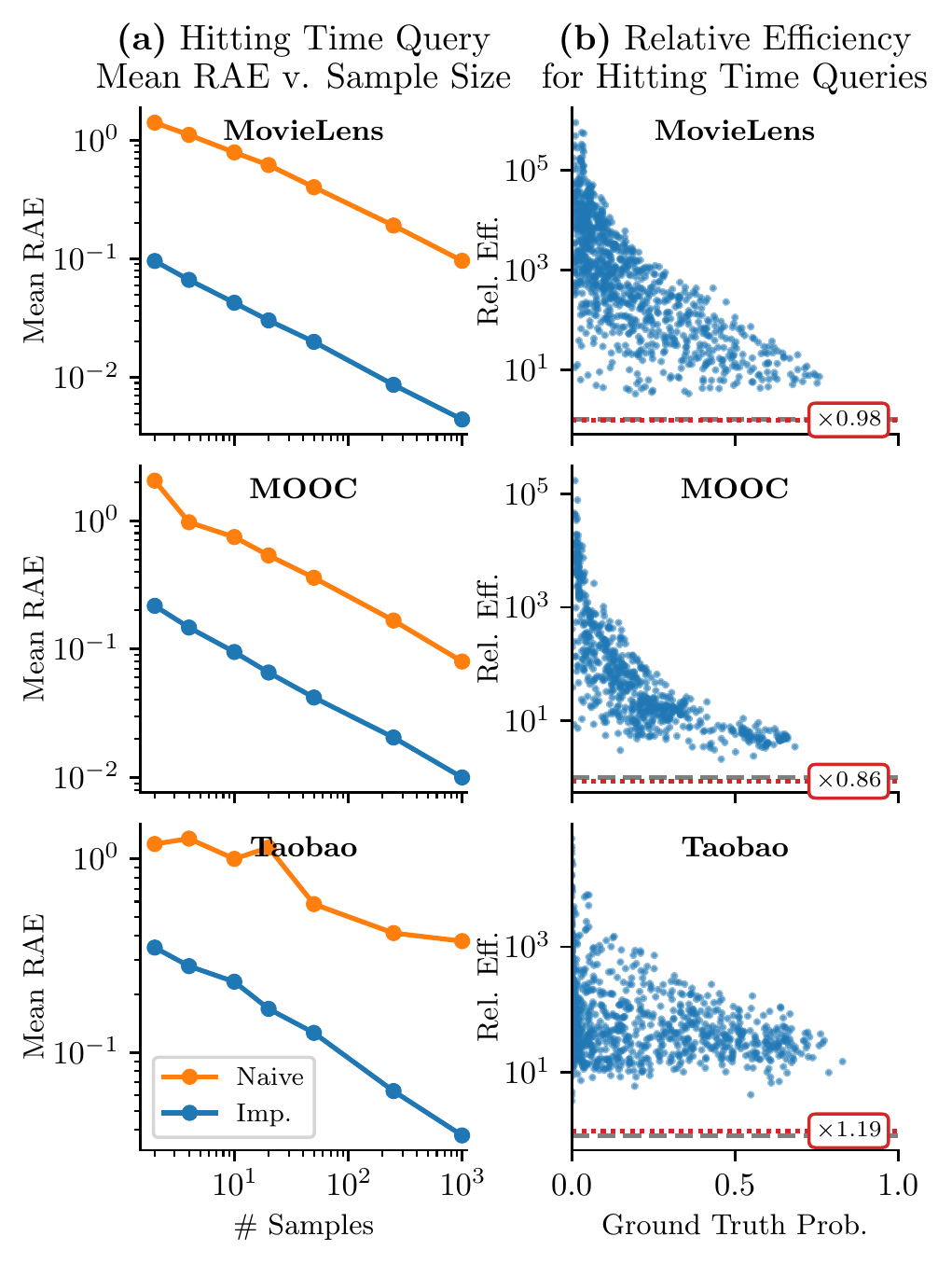}
    \caption{Results from 1000 different hitting time queries evaluated on models trained on three different datasets. (a) Average relative absolute error for naive and importance sampling shown in comparison to number of sampled sequences used. (b) Estimated relative efficiency values for importance sampling compared to naive sampling plotted against ground truth hitting time query values. Gray dashed lines indicate an efficiency of 1. Red lines with associated text box indicate the average multiplicative increase in computation time for importance sampling.}
    \label{fig:hit_plots}
\end{figure}

\paragraph{``A before B'' Queries}
Similar to the hitting time experiments, for ``A before B'' queries we similarly sample 1000 random test sequences and condition on the first five events $\hist[0,\tau_5]$. Then, we estimate the query $\prob(\hit(A) < \hit(B)\sep\hist[0,\tau_5])$ where $A$ and $B$ are randomly chosen to contain one third of the mark space $\mathbb{M}$.  

We compared estimating this query with naive sampling and importance sampling using varying amounts of samples: $\{2,4,10,25,50,250\}$. We utilized the truncated importance sampled estimator, \cref{eq:ab_estimator}, where $T$ is chosen dynamically for each sequence such that a maximum difference of 0.01 is allowed between the upper and lower bounds.
Mean RAE compared to ground truth (estimated using naive sampling with 5,000 samples) can be seen in \cref{fig:ab_err}.\footnote{Importance sampling would have been used for ground truth here; however, it is more sound to use an unbiased estimator for ground truth.} Like the hitting time results, we can see roughly an order of magnitude improvement in performance. Some results indicate that the limiting factor is the precision threshold for choosing $T$ (e.g., see MovieLens results). We also see a similar variance reduction relative to previous experiments, shown in \cref{fig:ab_eff}. Here, the runtime cost is much greater as we have to accumulate an integral over an indefinite amount of time; however, we can see that on average it is still very much ``worth it'' to utilize this framework over naive sampling as evidenced by all of the blue dots above the red line.

\begin{figure}[t]
    \centering{\phantomsubcaption\label{fig:ab_err}\phantomsubcaption\label{fig:ab_eff}}
    \includegraphics[width=\columnwidth]{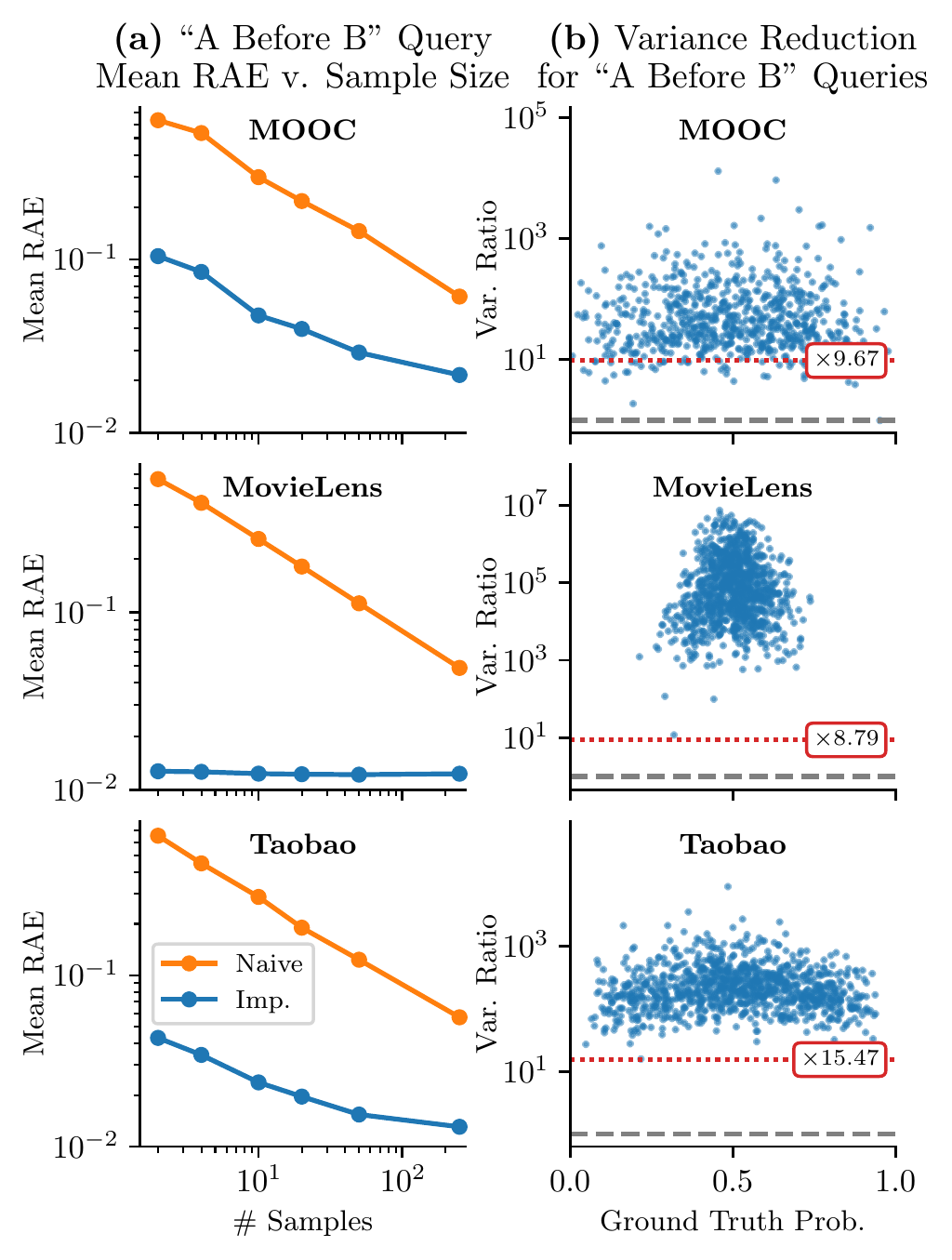}
    \caption{Same setup as seen in \cref{fig:hit_plots} with the same models and datasets only applied to ``A before B'' queries. Results for (b) are presented as ``variance reduction'' instead of ``relative efficiency'' since our derived estimator for importance sampling is biased due to truncating the integral in \cref{eq:ab_estimator}.}
    \label{fig:ab_plots}
\end{figure}

\paragraph{Marginal Mark Distribution Queries}
We additionally performed $n^\text{th}$ marginal mark distribution queries in much the same vein as the hitting time queries. Due to space limitations, the majority of the details and results can be found in the Appendix. That being said, we found that the resulting relative efficiencies for these queries to be much less than those of the other queries, but still more efficient than naive as \cref{thm:eff} suggests. Across the datasets, the median relative efficiency ranged from 1.9 to 2.7.
We speculate this to be due to the fact that the bounds of integration in the estimator are tied to sampled event times rather than being static values, inducing quite a lot of potential variance in the estimator. 

\subsection{Synthetic Experiments}
\begin{figure}
    \centering
    \includegraphics[width=\columnwidth]{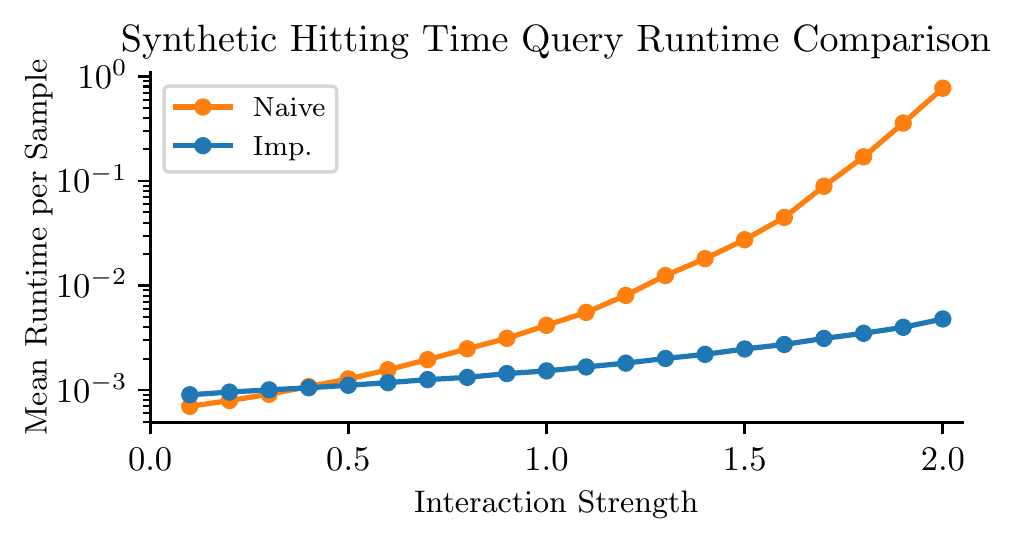}
    \caption{Average wall-clock time taken to generate a \underline{single} sample under naive and importance sampling for hitting time queries with 1,000 randomly instantiated parametric Hawkes models with different scalings of ``interaction strength'' (i.e., amount of modeled cross-mark interaction).}
    \label{fig:interaction_strength}
\end{figure}

For artificial experiments, we wanted to investigate the trends of our estimation procedures for a variety of queries over \emph{many different} models---something that is difficult to do with real-world data as we typically only have access to one model trained on a given dataset.
As such, we primarily focus on randomly instantiated parametric Hawkes processes with exponential kernels (see Appendix for details).
Under this setting, we were able to recreate similar findings in terms of estimation error and efficiency for the types of queries evaluated with real-world data. Due to space limitations these results can be seen in the Appendix.

In addition to these expected results, we also sought to investigate how different aspects of the underlying model affect the estimation procedure. In particular, we measured the average wall-clock time taken to generate samples for naive estimation and importance sampling as a function of how much cross-mark interaction is present. We modulate the \emph{interaction strength} in these generated models by changing the scale of the randomly generated mark-to-mark intensity parameters. The results can be seen in \cref{fig:interaction_strength} where we evaluated both estimation procedures on random hitting time queries for 1,000 different generated models with each across a span of interaction strengths. As more cross-mark excitement is encouraged by a model, the runtime it takes to sample a sequence over the same observation window becomes much longer in general; however, importance sampling counters this trend due to zeroing out (potentially several) marked intensities in $q$, thus barring events from happening. From these results, we can see that our importance sampling procedure is more robust to the underlying dynamics of a model over sampling windows fixed in time.

\section{Conclusion}
In this work, we proposed a general restricted-mark framework that enables us to efficiently answer a range of previously intractable probabilistic queries for continuous-time sequential event data. Experimental results show a significant improvement from importance sampling efficiency by several orders of magnitude compared to naive estimates, results that are consistent across three real-world datasets in different application domains with varying sequence lengths and numbers of marks. 
These results can in principle be further improved by producing more computationally efficient sampling procedures to use in conjunction with our proposed importance sampling estimators.

\subsubsection*{Acknowledgements}
This work was supported by National Science Foundation Graduate
Research Fellowship grant DGE-1839285, by an NSF CAREER Award, by the National Science Foundation under award numbers 1900644, 2003237, and 2007719, by the National Institute of Health under awards R01-AG065330-02S1 and R01-LM013344, by the Department of Energy under grant DE-SC0022331, by the HPI Research Center in Machine Learning and Data Science at UC Irvine, and by Qualcomm Faculty awards.

\bibliography{sample}

\onecolumn
\appendix

\section{Efficiency Proof Lemma}
\begin{lemma}
If a bounded random variable $X\in[0,1]$, with mean $\pi$ and CDF $F$, has $\var\left[X\right] = \pi(1-\pi)$, then $X\sim\text{Bern}(\pi)$.
\end{lemma}
\begin{proof}
Let $X$ be a random variable with support $[0,1]$, mean $\pi$, and variance $\pi(1-\pi)$. It then follows that:
\begin{align}
\var\left[X\right] & = \E\left[X^2\right] - \E\left[X\right]^2 \\
\implies \pi(1-\pi) & = \E\left[X^2\right] - \pi^2 \\
\implies \pi & = \E\left[X^2\right] \\
 & = \int_{[0,1]} x^2 dF(x) \\
 & = \int_{\{0,1\}} x^2dF(x) + \int_{(0,1)} x^2 dF(x) \\
 & = \prob(X=1)  + \int_{(0,1)} x^2 dF(x)
\end{align}
$\int_{(0,1)} x^2 dF(x) > 0$ if and only if $\prob(X \in (0,1)) > 0$. If we assume that $p(X \in (0,1)) > 0$, then it follows that:
\begin{align}
\pi & = \prob(X=1)  + \int_{(0,1)} x^2 dF(x) \\
& < \prob(X=1)  + \int_{(0,1)} x dF(x) \\
& = \prob(X=1) + (\pi - \prob(X=1)) \\
& = \pi,
\end{align}
however, $\pi \nless \pi$. Hence, by contradiction $\prob(X \in (0,1)) = 0$ which implies that $\prob(X=1)=\pi$ and $\prob(X=0)=1-\pi$ since $\E\left[X\right]=\pi$. Thus, it can be concluded that $X \sim \text{Bern}(\pi)$.
\end{proof}

\section{Deriving ``A Before B'' Estimator}
Let $A,B\subset\mathbb{M}$ and $A\cap B = \varnothing$. Recall that $\hist_A[0,t]$ is the sequence of events over times $[0,t]$ with the restriction that the marks must all belong to $A$. Finally, let $\q$ describe a proposal distribution with $\mu_k^*(t) = \ind(k \notin A \cup B)\lambda^*_k(t)$. With this in mind, we derive the expected value expression for the ``A Before B'' queries:
\begin{align}
\prob&\left(\hit(A) < \hit(B)\right) = \int_0^\infty \prob\left(\hit(A) < \hit(B), \hit(A) = t\right) dt \\
& = \int_0^\infty \sum_{k\in A}\prob\left(\hist [t,t]=\{(t,k)\}, \hist_A(t)=\varnothing, \hist_B(t)=\varnothing\right) dt \\
& = \int_0^\infty \sum_{k\in A}\prob\left(\hist [t,t]=\{(t,k)\}, \hist_{A\cup B}(t)=\varnothing\right) dt \\
& = \int_0^\infty \sum_{k\in A}\E_{p}\left[\prob\left(\hist [t,t]=\{(t,k)\}, \hist_{A\cup B}(t)=\varnothing\sep\hist(t)\right)\right] dt \displaybreak\\
& = \int_0^\infty \sum_{k\in A}\E_{\hist(t)\sim \prob}\left[\prob\left(\hist [t,t]=\{(t,k)\}\sep \hist_{A\cup B}(t)=\varnothing, \hist(t)\right)\prob\left(\hist_{A\cup B}(t)=\varnothing\sep \hist(t)\right)\right] dt \\
& = \int_0^\infty \sum_{k\in A}\E_{\hist(t)\sim \prob}\left[\prob\left(\hist [t,t]=\{(t,k)\}\sep \hist(t)\right)\ind\left(\hist_{A\cup B}(t)=\varnothing\right)\right] dt \\
& = \int_0^\infty \sum_{k\in A}\E_{\hist(t)\sim \prob}\left[\prob\left(\hist [t,t]=\{(t,k)\}\sep \hist(t)\right)\ind\left(\hist_{A\cup B}(t)=\varnothing\right)\right] dt \\
& = \int_0^\infty \sum_{k\in A}\E_{\hist(t)\sim \prob}\left[\lambda^*_k(t)\ind\left(\hist_{A\cup B}(t)=\varnothing\right)\right] dt \\
& = \int_0^\infty \E_{\hist(t)\sim \prob}\left[\lambda^*_A(t)\ind\left(\hist_{A\cup B}(t)=\varnothing\right)\right] dt \\
& = \int_0^\infty \E_{\hist(t)\sim \q}\left[\lambda^*_A(t)\ind\left(\hist_{A\cup B}(t)=\varnothing\right)\frac{\prob\left(\hist(t)\right)}{\q\left(\hist(t)\right)}\right] dt \\
& = \int_0^\infty \E_{\hist(t)\sim \q}\left[\lambda^*_A(t)\exp\left(-\int_0^t \lambda^*_{A \cup B}(s) ds\right)\right] dt \\ 
& = \int_0^\infty \E_{\hist(\infty)\sim \q}\left[\lambda^*_A(t)\exp\left(-\int_0^t \lambda^*_{A \cup B}(s) ds\right)\right] dt \\ 
& = \E_{\hist(\infty)\sim \q}\left[\int_0^\infty \lambda^*_A(t)\exp\left(-\int_0^t \lambda^*_{A \cup B}(s) ds\right)\right] dt 
\end{align}
where the last line is justified due to the Dominated Convergence Theorem. The prerequisites for this theorem are satisfied by noting that:
\begin{align}
\int_0^\infty \lambda^*_{A}(t)\exp\left(-\int_0^t \lambda^*_{A \cup B}(s) ds\right)dt & \leq  \int_0^\infty \lambda^*_{A\cup B}(t)\exp\left(-\int_0^t \lambda^*_{A \cup B}(s) ds\right)dt \\
& = -\int_0^\infty \frac{d}{dt}\exp\left(-\int_0^t \lambda^*_{A \cup B}(s) ds\right)dt \\
& = \exp\left(-\int_0^0 \lambda^*_{A \cup B}(s) ds\right) - \exp\left(-\int_0^\infty \lambda^*_{A \cup B}(s) ds\right) \\
& = 1 - \exp\left(-\int_0^\infty \lambda^*_{A \cup B}(s) ds\right) \\
& \leq 1.
\end{align}

\section{Further Experimental Details and Results}

\subsection{Dataset Preprocessing}
We evaluate our methods for probabilistic querying on three real-world user-behavior datasets in different application domains that are publicly available. All datasets do not include personally identifiable information, where users are identified by unique integer IDs. For all our experiments, sequences are defined as the event histories of each user, where events have timestamps in seconds. We changed the time resolution from seconds to hours for better interpretability of our query implications. Additionally, we only consider sequences with at least 5 events and at most 200 events. We use 75\% of the sequences for training, 10\% for validation, and 15\% for testing.

\paragraph{MovieLens}
The MovieLens 25M dataset \citep{harper2015movielens} contains 25 million movie ratings by 162,000 users. The movie category (genre) associated with each rating is modeled as marks, and the exact rating value is ignored.\footnote{A single movie in this dataset can possibly have multiple categories associated with it. To accommodate this, if a movie has multiple categories we randomly select a subset of two categories to represent the movie. Note this highlights the benefits of formulating queries as sets of marks instead of just singular marks. To evaluate the hitting time of the next ``comedy'' movie reviewed, then we would need to evaluate the hitting time of the set of all pairs of categories where one element is the comedy genre. This is essentially describing marginalizing over a hierarchical structure for the marks.} For each sequence, the start and the end time are defined as the first and the last event time of each user respectively, because the time span for different users ranges from seconds to years. The first event is discarded in the sequence of history and is only used to indicate $t=0$. For consistent dynamics across the dataset, we filter the data to only contain reviews at or after the year 2015. This leaves 34,935 remaining sequences, each from a unique user.

\paragraph{MOOC}
The MOOC user action dataset \citep{kumar2019predicting} represents user activities on a massive open online course (MOOC) platform. It consists of 411,749 course activities in 97 different types modeled as marks for 7,047 users, out of which 4,066 users dropped out after an activity. Timestamps are standardized to start from timestamp 0. We use the last event time for drop-out users as the end of their sequences, and the maximum timestamp for the other users. 

\paragraph{Taobao}
The Taobao user behavior dataset \citep{zhu2018learning} was originally intended for recommendations for online shopping, which includes four behaviors: page viewing, purchasing, adding items to the chart, and to wishlist. We focus on page viewing of users as events, and model the item category as the event mark, which has marketing implications such as click through rate of recommending some types of items. Due to the large scale of the dataset, we use a subset of 2,000,000 events on 8 consecutive calendar days inclusive (November 25th, 2017 - December 2nd, 2017), as well as the most frequent 1,000 marks (item categories) to demonstrate query answering. All user sequences have the same length.

\subsection{Modeling Details}
For each of the real-world datasets, a neural Hawkes process model \citep{mei2017neural} was trained with a batch size of 128, a learning rate of 0.001, a linear warm-up learning rate schedule over the first 1\% of training iterations, a max allowed gradient norm of $10^4$ for training stability, and the Adam stochastic gradient optimization algorithm \citep{KingmaB14} with default hyperparameters. Specific datasets had specific model hyperparameters due to differences in the amount of data and total possible marks. The details for these can be found in \cref{tab:model_details}. All models were trained for a fixed amount of epochs; however, each one was confirmed to have converged based on average held-out validation log-likelihood.

\begin{table}[H]
    \centering
    \caption{Model Hyperparameters for Real-World Datasets}
    \begin{tabular}{lccc}
    \toprule
    Hyperparameter & MovieLens & MOOC & Taobao \\
    \midrule
    \# Training Epochs & 100 & 100 & 300 \\
    Mark Embedding Size & 32 & 32 & 64 \\
    Recurrent Hidden State Size & 64 & 64 & 128 \\
    \bottomrule
    \end{tabular}
    \label{tab:model_details}
\end{table}

\subsection{Integration Approximation}
For the real-world experiments, many integrals need to be evaluated in order to produce estimates for various queries. Since we use essentially black-box MTPP models, we do not have access to an analytical form for integration. As such, we must estimate every integral at play.

To do this, we utilize the trapezoidal rule. For reference, this involves estimating integrals with the following summation:
\begin{align}
\int_a^b f(x)dx \approx \sum_{i=1}^N \left(f(x_i)+f(x_{i-1})\right)\frac{x_i-x_{i-1}}{2}
\end{align}
where the points $x_{i-1} < x_i$ span the interval $[a,b]$ with $x_0=a$ and $x_N=b$. For hitting time queries and marginal mark queries, we utilize $N=1000$ integration points with equal spacing. It is likely that we could get by with much less for these queries, however, for the sake of high precision for experimental results we utilized a large amount of sample points.

For the ``A Before B'' queries, we found that the resolution at which the estimator is evaluated at is of much more importance than the other queries. As such, for this query we estimate integrals in an online fashion during the sampling procedure for each proposal distribution sample sequence in conjunction with a very high proposal dominating rate (see \citet{ogata1981lewis} for details). This allowed for a much more efficient procedure (in both computation and memory consumption) compared to integrating results after sampling. 

\subsection{Marginal Mark Query Experiments}
Similar to the hitting time experiments, for the marginal mark queries we similarly sample 1000 random test sequences and condition on the first five events $\hist[0,\tau_5]$. Then, we estimate the query $\prob(\kappa_{8}\in A\sep\hist[0,\tau_5])$ where $A$ is a randomly selected subset of all of the unique marks that appear in the entire sequence $\hist$. This is done to ensure that $A$ contains relevant marks for the given sequence.  

\begin{figure}[H]
    \centering{\phantomsubcaption\label{fig:marg_mark_err}\phantomsubcaption\label{fig:marg_mark_eff}}
    \includegraphics{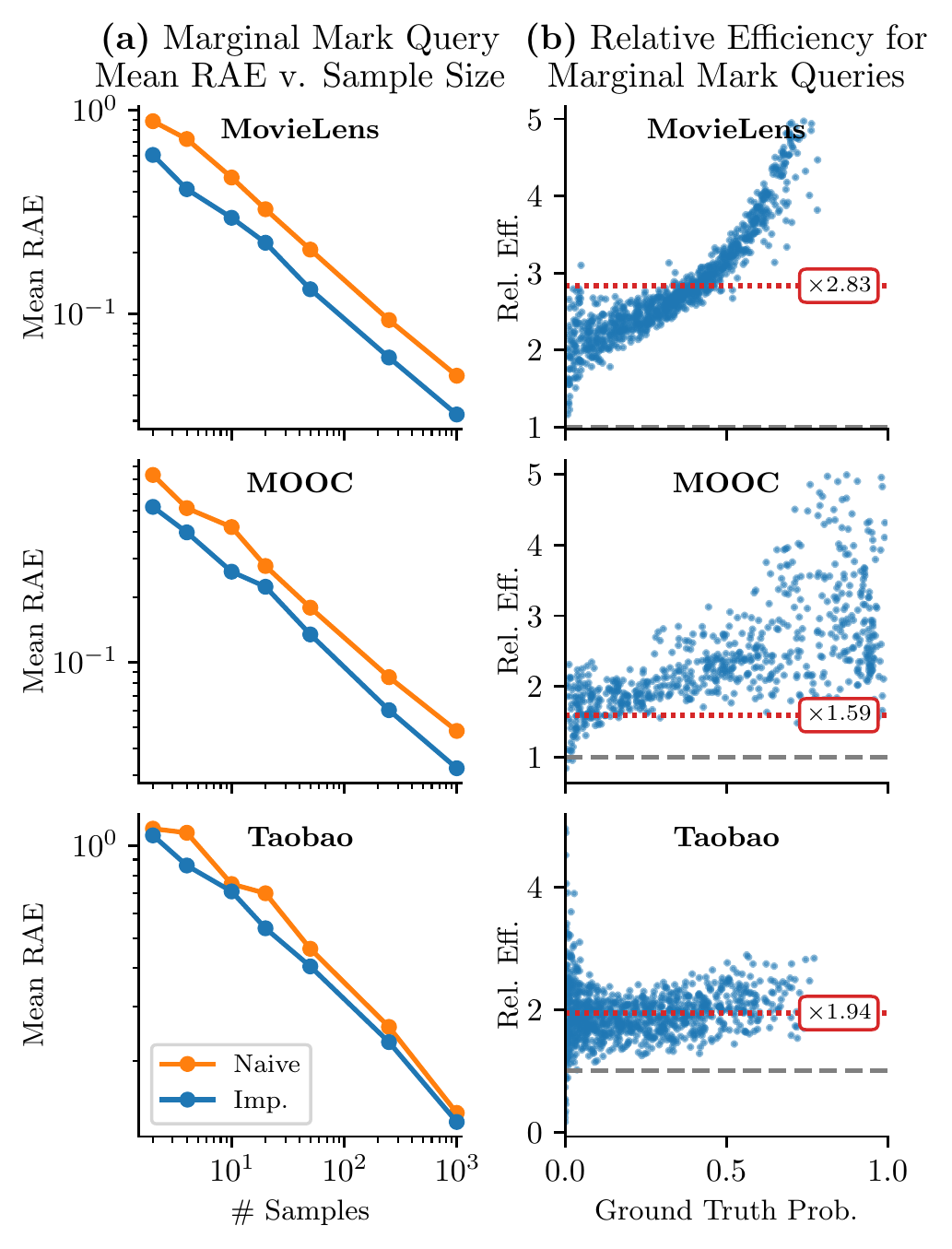}
    \caption{Results from 1,000 different marginal mark queries evaluated on models trained on three different datasets. (a) Average relative absolute error for naive and importance sampling shown in comparison to number of sampled sequences used. (b) Estimated relative efficiency values for importance sampling compared to naive sampling plotted against ground truth marginal mark query values. Gray dashed lines indicate an efficiency of 1. Red lines with associated text box indicate the average multiplicative increase in computation time for importance sampling.}
    \label{fig:marg_mark_plots}
\end{figure}

We compared estimating this query with naive sampling and importance sampling using varying amounts of samples: $\{2,4,10,25,50,250,1000\}$. Mean RAE compared to ground truth (estimated using importance sampling with 5,000 samples) can be seen in \cref{fig:marg_mark_err}. We witness roughly $1.5$ to $3$ times improvement in performance for the same amount of samples. 
Similar to hitting time query results, we attribute this improvement to the fact that naive sampling only collects binary values, whereas our proposed procedure collects much more dense information over the entire span from $\tau_5$ to $\tau_8\sim q$.

We also analyze the relative efficiency of our estimator compared to naive sampling. For each query asked, the efficiency was estimated using 5,000 importance samples. The results can be seen in \cref{fig:marg_mark_eff}.  
We achieve a decent decrease in variance, in the majority of contexts, across all datasets. Like the hitting time query results, we also note a pretty strong correlation between underlying ground truth values and the relative efficiency of this estimator. 

Notably, these results don't appear to be as drastic as the hitting time query results. We believe this is due to the fact that the estimator's bounds of integration are sampled from the proposal distribution to be between $\tau_{N-1}$ and $\tau_N$ for each sequence (whereas the bounds for the hitting time query $p(\hit(k)\leq t)$ is always the span of $[0,t]$). This added variability seems to dampen the impact of the integration in the first place.

\subsection{Synthetic Experiments}
We also perform experiments on hitting time queries and ``A Before B'' queries using self-exciting parametric Hawkes processes \citep{hawkes1971spectra} with exponential kernels. The intensity has the explicit form:
\begin{align}
    \lambda_k^*(t) &= \mu_k + \sum_{\kappa=1}^K \int_0^t \phi_{k \kappa}(t - u) d N_\kappa(u) \label{eq:mutual_density}\\
    &= \mu_k + \sum_{\kappa=1}^K \sum_{\tau_{\kappa, i} < t} \phi_{k \kappa}(t - \tau_{\kappa, i}) \label{eq:mutual_density2},
\end{align}
where $\tau_{\kappa, i}$ refers to the time when the $i$th event of type $\kappa$ occurs, $\phi(\boldsymbol{x}) = \boldsymbol{\alpha} e^{-\boldsymbol{\beta} \boldsymbol{x}}$ with $\boldsymbol{\alpha}, \boldsymbol{\beta} > \boldsymbol{0}$, and Equation \ref{eq:mutual_density2} can be expressed in matrix form. The first term $\boldsymbol{\mu}$ is referred to as the \textit{base intensity} or \textit{background intensity} in literature. Each event instantaneously increases the intensity by corresponding $\alpha$ and its influence decays exponentially at $\beta$ over time. The total intensity of all $K$ types of events is
\begin{align}
    \lambda^*(t) = \sum_{k=1}^K \lambda_k^*(t). \label{eq:hawkes_total_intensity}
\end{align}
Under this parametric form, the integrals for query estimates can be computed in closed forms. 

We evaluate our methods on (i) hitting time queries $p(\hit(k)\leq t)$
and (ii) ``A before B'' queries $p(\hit(A) < \hit(B))$. All results are averaged over 1,000 different randomly initiated parametric self-exciting Hawkes models that are not feasible for real-world datasets. These random models have different total amounts of marks ranging from $K=3$ to $K=10$, and have different inter-event effects as well as exponential rates of decay. We use 10 integration points for hitting time queries and 1,000 integration points for ``A Before B'' queries. \footnote{For the ``A Before B'' queries, using 1,000 integration points after sampled provided sufficient precision and we did not need to employ the online integration approach used with the real-world experiments. This is most likely attributable to the well-behaved dynamics exhibited by the classic parametric Hawkes intensity. This is also why we used a reduced amount of integration points for the synthetic hitting time queries as well compared to the real-world experiments.} 

For each hitting time query, we fix $t=1$ and $k=0$, because the model is randomly generated.
For the ``A Before B'' queries, like the real-world experiments we let them be randomly sampled subsets of the vocabulary such $|A|=|B|\approx K/3$.
We evaluate the hitting time queries using varying amounts of samples: $\{2,4,10,25,50,250,1000\}$. For ``A before B'' queries, we only use $\{2,4,10,25,50,250\}$ number of samples because the query estimates take longer. Ground truth probabilities are calculated using 5,000 samples with importance sampling for hitting time queries and with naive method for ``A before B'' queries respectively. The plots \cref{fig:supp_hawkes_hitting} and \cref{fig:supp_hawkes_ab_overall} reveal the similar patterns that our method is more efficient than the naive estimates averaged over a range of different model settings.\footnote{We note that in preliminary synthetic experiments on hitting time queries, we did observe similarly drastic relative efficiency values for hitting time queries with lower ground truth probability as seen in the real-world experiments.}

\begin{figure}[H]
    \centering{\phantomsubcaption\label{fig:supp_hawkes_hit}\phantomsubcaption\label{fig:supp_hawkes_hit_eff}}
    \includegraphics{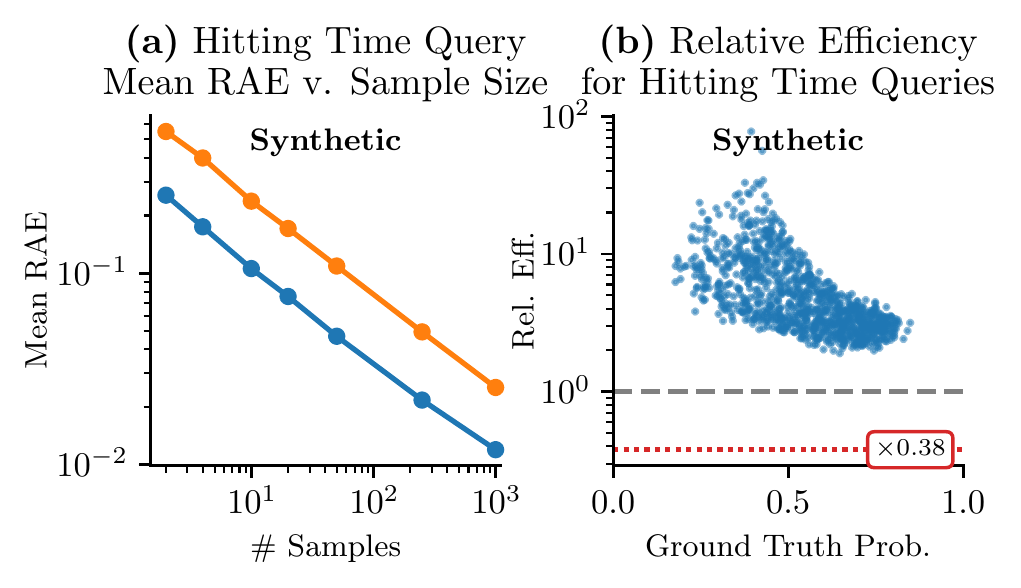}
    \caption{Synthetic experiments for hitting time queries evaluated on parametric self-exciting Hawkes processes. (a) Average relative absolute error for naive and importance sampling shown in comparison to number of sampled sequences used. (b) Estimated relative efficiency values for importance sampling compared to naive sampling plotted against ground truth hitting time query values. Gray dashed lines indicate an efficiency of 1. Red lines with associated text box indicate the average multiplicative increase in computation time for importance sampling.
    }
    \label{fig:supp_hawkes_hitting}
\end{figure}

\begin{figure}[H]
    \centering{\phantomsubcaption\label{fig:supp_hawkes_ab}\phantomsubcaption\label{fig:supp_hawkes_ab_eff}}
    \includegraphics{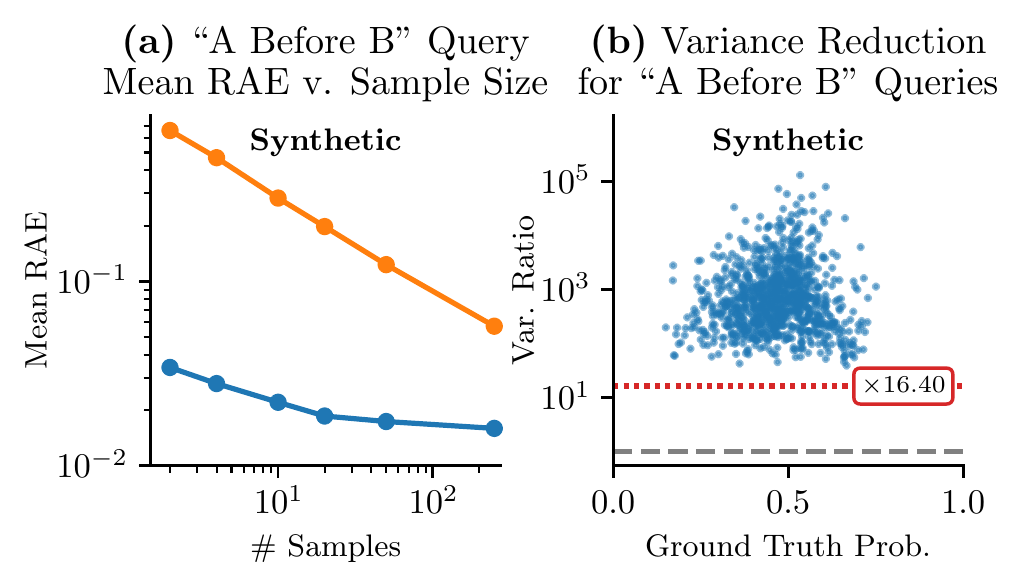}
    \caption{Synthetic experimental results evaluated on 1,000 different random models and ``A before B'' queries for parametric self-exciting Hawkes processes. (a) Average relative absolute error for naive and importance sampling shown in comparison to number of sampled sequences used. (b) Estimated relative efficiency values for importance sampling compared to naive sampling plotted against ground truth ``A before B'' query values. Gray dashed lines indicate an efficiency of 1. Red lines with associated text box indicate the average multiplicative increase in computation time for importance sampling.}
    \label{fig:supp_hawkes_ab_overall}
\end{figure}

\end{document}